\documentclass[conference,letterpaper]{IEEEtran}

\usepackage[T1]{fontenc}
\usepackage{url}
\usepackage{ifthen}
\usepackage{cite}
\usepackage[cmex10]{amsmath}
\usepackage{float}

\usepackage{amsmath,amssymb}
\usepackage{esvect} 
\usepackage[mathscr]{euscript} 
\usepackage{amsthm} 
\usepackage{tikz} 
\usepackage{graphicx} 
\newcommand*{\Scale}[2][4]{\scalebox{#1}{\ensuremath{#2}}} 
\usepackage{framed} 
\usepackage{dblfloatfix} 
\usepackage{multirow}
\usepackage{color}
\usepackage{algorithmic}
\usepackage{subfigure}
\usepackage{longtable}
\usepackage[hidelinks, breaklinks=true]{hyperref} 
\usepackage{xcolor}  \definecolor{shadecolor}{rgb}{.95,.95,.95}  
\usepackage[font=footnotesize]{caption}
\usepackage{nicefrac}
\usepackage{dsfont} 
\usepackage[export]{adjustbox} 
\usepackage{cancel} 
\usepackage{enumitem}	
\usepackage{wrapfig}
\usepackage[utf8]{inputenc}
\usepackage{comment}
\usepackage{soul} 
\usepackage{pifont} 

\usepackage[ruled,vlined]{algorithm2e}

\theoremstyle{definition}

\newtheorem{myTheorem}{Theorem}
\newtheorem{myLemma}{Lemma}
\newtheorem{myCorollary}{Corollary}

\newtheorem{myExample}{Example}

\renewenvironment{shaded*}{%
\MakeFramed{\advance\hsize-\width\FrameRestore}%
\leftskip 1.5em \rightskip 1.5em}%
{\endMakeFramed}%

\renewenvironment{shaded}{%
\MakeFramed{\advance\hsize-\width\FrameRestore}%
\leftskip -.5em \rightskip -.5em}%
{\endMakeFramed}%

\def \R{\mathbb{R}}
\def \V{\mathbf{V}}

\def \Gr{\mathbb G}

\def \1{{\mathds{1}}}
\def \T{\mathsf{T}}

\def \spn{{\rm span}}
\def \rnk{{\rm rank}}

\def \<{\langle}
\def \>{\rangle}

\def \Fr{{{\rm F}}}

\def \r{{\rm r}}
\def \m{{\rm m}}
\def \N{{\rm N}}

\def \x{{\boldsymbol{{\rm x}}}}
\def \a{{\boldsymbol{{\rm a}}}}
\def \b{{\boldsymbol{{\rm b}}}}

\def \al{{\boldsymbol{\alpha}}}

\def \I{{\boldsymbol{{\rm I}}}}
\def \U{{\boldsymbol{{\rm U}}}}
\def \V{{\boldsymbol{{\rm V}}}}
\def \W{{\boldsymbol{\hat{{\rm U}}}}}
\def \A{{\boldsymbol{{\rm A}}}}
\def \B{{\boldsymbol{{\rm B}}}}

\def \Z{{\boldsymbol{{\rm Z}}}}
\def \P{{\boldsymbol{{\rm P}}}}
\def \M{{\boldsymbol{{\rm M}}}}

\def \i{{\rm i}}
\def \j{{\rm j}}

\def \o{{\Omega}}
\def \O{{\boldsymbol{\Omega}}}

\def \sU{{\mathbb{U}}}
\def \sV{{\mathbb{V}}}
\def \sW{{\hat{\mathbb{U}}}}
\def \dist{{d}}

\begin{document}
\title{A Perturbation Bound on the Subspace Estimator from Canonical Projections} 

\author{%
  \IEEEauthorblockN{Karan Srivastava, Daniel Pimentel-Alarc\'on}
  \IEEEauthorblockA{University of Wisconsin-Madison\\
                    Department of Mathematics, Department of Biostatistics, Wisconsin Institute for Discovery \\
                    Email: \{ksrivastava4, pimentelalar\}@wisc.edu}
}

\maketitle

\begin{abstract}
This paper derives a perturbation bound on the optimal subspace estimator obtained from a subset of its canonical projections contaminated by noise. This fundamental result has important implications in matrix completion, subspace clustering, and related problems.
\end{abstract}

\section{Introduction}

 This paper presents a perturbation bound on the optimal subspace estimator obtained from noisy versions of its canonical projections. More precisely, let $\sU \subset \R^\m$ be an $\r$-dimensional subspace in general position. Given $\o_\i \subset \{1,\dots,\m\}$, let $\sU^\o_\i \subset \R^{|\o_\i|}$ denote the projection of $\sU$ onto the canonical coordinates in $\o_\i$ (i.e. the usual coordinate planes with the standard basis). Let $\U^\o_\i \in \R^{|\o_\i| \times \r}$ be a basis of $\sU^\o_\i$, let $\Z^\o_\i$ denote a $|\o_\i|\times \r$ noise matrix, and define
\begin{align}
\label{vEq}
\V^\o_\i \ = \ \U^\o_\i + \Z^\o_\i.
\end{align}
Given $\V^\o_1,\dots,\V^\o_\N$, our goal is to estimate $\sU$ (see Figure \ref{intuitionFig}, where $\sV^\o_\i:=\spn\V^\o_\i$).

\begin{figure}[b]
\centering
\includegraphics[width=3.45cm]{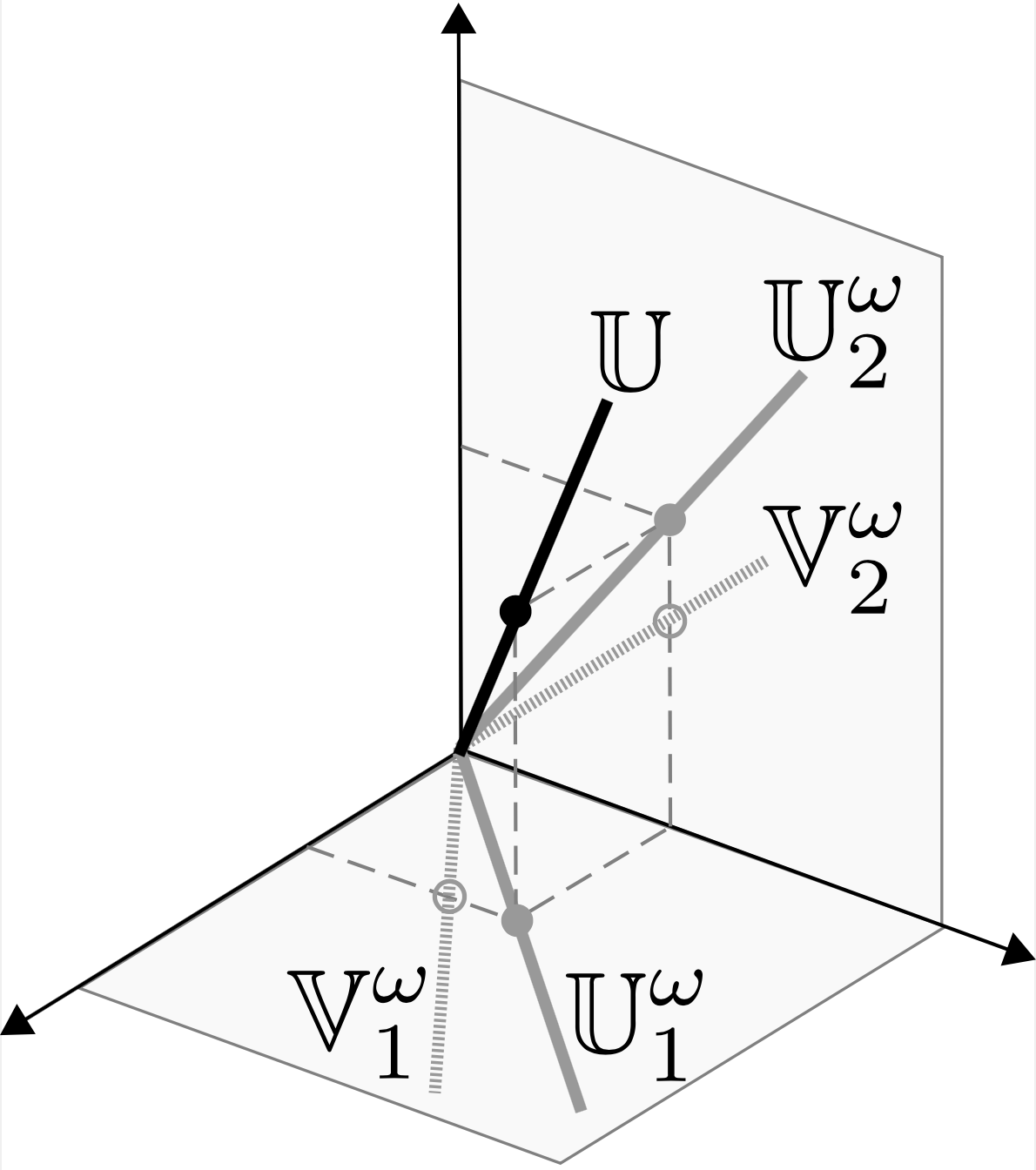}
\caption{When can $\sU$ be identified from the {\em noisy} canonical projections $\{\sV^\o_\i\}$?}
\label{intuitionFig}
\end{figure}

\begin{myExample}
\begin{em}
Consider the subspace $\sU$ spanned by $\U$ below, and suppose we only observe its projections $\U^\o_1,\U^\o_2,\U^\o_3$:
\begin{align*}
\small
\U = \left[
\begin{matrix}
1 & 1\\
1 & 2\\
1 & 3\\
1 & 4\\
1 & 5\\
\end{matrix}\right],
\begin{matrix}
\U^\o_1=\left[ \begin{matrix} 1 & 1 \\ 1 & 2 \\ 1 & 3 \end{matrix}\right]
\\ \\ \\
\end{matrix},
\begin{matrix}
\\ 
\U^\o_2=\left[ \begin{matrix} 2 & 2 \\ 2 & 3 \\ 2 & 4 \end{matrix}\right]
\\ \\
\end{matrix},
\begin{matrix}
\\ \\
\U^\o_3=\left[ \begin{matrix} 3 & 3 \\ 3 & 4 \\ 3 & 5 \end{matrix}\right]
\end{matrix}.
\end{align*} 
Here the first projection $\U^\o_1$ only includes the first three coordinates (i.e., $\o_1=\{1,2,3\}$), the second projection $\U^\o_2$ only includes the three middle coordinates ($\o_2=\{2,3,4\}$), and the third projection only includes the last three coordinates ($\o_3=\{3,4,5\}$). Notice that the bases themselves may differ, but their spans (onto the projected coordinates) agree with $\sU$. Given noisy versions of $\U^\o_1,\U^\o_2,\U^\o_3$ (which we call $\V^\o_1,\V^\o_2,\V^\o_3$; see \eqref{vEq}), our goal is to estimate $\sU$.
\end{em}
\end{myExample}

Estimation of this sort is only possible under certain conditions on the observed projections \cite{pimentel2015deterministic}. In this paper, we will assume without loss of generality that:
\begin{shaded*}
\begin{itemize}
\item[\textbf{C1.}]
Each projection includes exactly $\r+1$ coordinates, i.e., $|\o_\i|=\r+1$.

\item[\textbf{C2.}]
There is a total of $\N=\m-\r$ available projections.

\item[\textbf{C3.}]
Every subset of $\ell$ projections includes at least $\ell+\r$ coordinates.

\item[\textbf{C4.}]
$\Z^\o_\i$ is drawn independently according to a density $\nu$ such that its spectral norm satisfies $\|\Z^\o_\i\|_2 \leq \epsilon$.

\item[\textbf{C5.}]
The smallest singular value of $\V^\o_\i$ is lower bounded by $\delta$.
\end{itemize}
\end{shaded*}

To better understand these conditions, first notice that since $\sU$ is $\r$-dimensional and in general position, if $|\o_\i|\leq\r$, then $\sV^\o_\i=\R^{|\o_\i|}$, in which case $\sV^\o_\i$ provides no information about $\sU$. Hence, rather than an assumption, \textbf{C1} is a fundamental requirement that describes the minimal sampling conditions on each projection. Under this minimal sampling regime, \textbf{C2}-\textbf{C3} are strictly necessary to reconstruct $\sU$ from the {\em noiseless} projections $\{\sU^\o_\i\}$ \cite{pimentel2015deterministic}. So, rather than assumptions, \textbf{C2}-\textbf{C3} are also fundamental requirements (see Theorem 3 in \cite{pimentel2015deterministic} for a simple generalization to denser samplings). Condition \textbf{C4} simply requires that the noise is bounded by $\epsilon$. Similarly, \textbf{C5} requires that the signal {\em power} in each projection is at least $\delta$. This way, $\delta/\epsilon$ can be thought of as the signal-to-noise ratio.

\textbf{In this paper} we bound the error of the optimal estimator $\sW$ of $\sU$ derived in \cite{pimentel2015deterministic}. Such estimator $\sW$ is \textit{optimal} in the sense that it is the only subspace of the same dimension as $\sU$ that perfectly {\em agrees} with the observed projections (i.e., $\sW^\o_\i=\sV^\o_\i$ for every $\i$). To construct $\sW$, first observe that since $\sU$ is $\r$-dimensional and in general position, and $|\o_\i|=\r+1$, $\sU^\o_\i$ is a hyperplane (i.e., an $\r$-dimensional subspace of $\R^{\r+1}$). Then so is $\sV^\o_\i$ $\nu$-almost surely. Let $\b^\o_\i \in \R^{\r+1}$ denote a normal vector orthogonal to $\sV^\o_\i$, i.e., a basis of the null-space of $\sV^\o_\i$. Define $\b_\i$ as the vector in $\R^\m$ with the entries of $\b^\o_\i$ in the coordinates of $\o_\i$ and zeros elsewhere. Finally, let $\B=[\b_1,\dots,\b_\N]$. It is easy to see that by construction, any $\r$-dimensional subspace in $\ker \B^\T$ agrees with the projections $\{\sV^\o_\i\}$. Moreover, since the samplings $\{\o_\i\}$ satisfy \textbf{C1}-\textbf{C3}, by Theorem 1 in \cite{pimentel2015deterministic} there is only one such subspace. Consequently, $\ker \B^\T$ is the only $\r$-dimensional subspace of $\R^\m$ that agrees with the projections $\{\sV^\o_\i\}$. In other words, the \textit{optimal} estimator of $\sU$ is given by
\begin{align}
\label{estimatorEq}
\sW \ = \ \ker\B^\T.
\end{align}
Our perturbation bound is measured in terms of the {\em chordal} distance over the Grassmannian between the {\em true} subspace $\sU$ and its optimal estimate $\sW$, defined as follows \cite{ye2016schubert}:
\begin{align*}
\dist_\Gr(\sU,\sW) \ &:= \ \frac{1}{\sqrt{2}} \ \|\P_\U - \P_\W \|_\Fr,
\end{align*}
where $\P_\U,\P_\W$ are the projection operators onto $\sU,\sW$, and $\|\boldsymbol{\cdot}\|_\Fr$ denotes the Frobenius norm. We use $\sigma(\boldsymbol{\cdot})$ to denote the {\em smallest} singular value.

With this, we are ready to present our main result, which bounds the optimal estimator error:
\begin{framed}
\begin{myTheorem}
\label{mainThm}
Let \textbf{C1}-\textbf{C5} hold. Then $\nu$-almost surely,
\begin{align}
\label{mainBound}
\dist_\Gr(\sU,\sW) \ \leq \ \frac{\epsilon \sqrt{2\r(\m-\r)} }{\delta \sigma(\B)}.
\end{align}
\end{myTheorem}
\end{framed}
The proof of Theorem \ref{mainThm} is in Section \ref{proofSec}. In words, Theorem \ref{mainThm} shows that the error of the optimal estimate $\sW$ is bounded by the projection's noise-to-signal ratio $\epsilon/\delta$, except for a factor $\sqrt{2\r(\m-\r)}$ related to the degrees of freedom of an $\r$-dimensional subspace of $\R^\m$, and a term $\sigma(\B)$ that depends on the particular sampling and the orientation of $\sU$, as discussed in Section \ref{orientationSec}.
Therefore, given observed data $\B$ and signal-to-noise ratio, one can compute an upper bound to the error in the noisy subspace estimator. 

\begin{myExample}
Continuing from Example 1, if we observe noisy projections 
\begin{align*}
\begin{matrix}
\V^\o_1=\left[ \begin{matrix} 1.1 & 0.9 \\ 1.2 & 1.8 \\ 0.9 & 3.2 \end{matrix}\right]
\\ \\ \\
\end{matrix},
\begin{matrix}
\\ 
\V^\o_2=\left[ \begin{matrix} 2.1 & 2.1 \\ 1.9 & 3.1 \\ 2.1 & 3.8 \end{matrix}\right]
\\ \\
\end{matrix},
\begin{matrix}
\\ \\
\V^\o_3=\left[ \begin{matrix} 3.1 & 3.1 \\ 2.8 & 4.2 \\ 3.1 & 4.9 \end{matrix}\right]
\end{matrix}.
\end{align*} 
with a signal-to-noise ratio of $0.1$. In this case, the smallest singular value of $\B$ is $0.33$ and our bound is $0.59$, whereas the true distance between $\sU$ and $\sW$ is $0.29$.
\end{myExample}

\section{Motivation}
Learning low-dimensional structures that approximate a high-dimensional dataset is one of the most fundamental problems in science and engineering. However, in a myriad of modern applications, data is missing in large quantities. For example, even the most comprehensive metagenomics database \cite{benson2016genbank} is highly incomplete \cite{cribdon2020pia}; most drug-target interactions remain unknown \cite{zhang2019predicting}; in image inpainting the values of some pixels are missing due to faulty sensors and image contamination \cite{mairal2009online}; in computer vision features are often missing due to occlusions and tracking algorithms malfunctions \cite{vidal2008multiframe}; in recommender systems each user only rates a limited number of items \cite{park2015preference}; in a network, most nodes communicate in subsets, producing only a handful of all the possible measurements \cite{balzano2012high}.

These and other applications have motivated the vibrant field of matrix completion, which aims to fill the missing entries of a data matrix, and infer its underlying structure. Perhaps the most studied case in this field is that of low-rank matrix completion (LRMC) \cite{candes2009exact, candes2010power, recht2011simpler}, which assumes the data lies in a linear subspace. Another popular case is high-rank matrix completion (HRMC) \cite{balzano2012high}, which allows samples (columns) to lie in a union of subspaces. Mixture matrix completion (MMC) \cite{pimentel2018mixture} extends this idea to a full mixture (columns and rows) of low-rank matrices. More recently, low-algebraic-dimension matrix completion (LADMC) \cite{pimentel2017low, ongie2021tensor} further generalizes these assumptions to allow data in non-linear algebraic varieties. There are numerous variants and generalizations of these models, such as multi-view matrix completion (MVMC) \cite{ashraphijuo2017characterization1}, low-Tucker-rank tensor completion (LTRTC) \cite{ashraphijuo2017characterization2}, and low-CP-rank tensor completion (LCRTC) \cite{ashraphijuo2017fundamental}.

\ul{Learning a subspace from its canonical projections plays a central role in each of these cases}. In particular, the fundamental conditions specified in Theorem 1 in \cite{pimentel2015deterministic} for the noiseless case are of particular importance to derive:

\begin{itemize}[leftmargin=.5cm]
\item[\ding{182}]
Deterministic sampling conditions for unique completability in LRMC (Theorem 2, Lemma 8 in \cite{pimentel2016characterization}).

\item[\ding{183}]
The information-theoretic requirements and sample complexity of HRMC (Theorems 1, 2 in \cite{pimentel2016information}).

\item[\ding{184}]
The fundamental conditions for learning mixtures in MMC (Theorem 1 in \cite{pimentel2018mixture}).

\item[\ding{185}]
Identifiability conditions to learn tensorized subspaces in LADMC (Theorem 1 in \cite{pimentel2017low}, Lemmas 2, 3 in \cite{ongie2021tensor}).

\item[\ding{186}]
A characterization of uniquely completable patterns for MVMC (Theorem 3 in \cite{ashraphijuo2017characterization1}, Theorem 4 in \cite{ashraphijuo2017deterministic}).

\item[\ding{187}]
Unique completability conditions for LTRTC (Lemma 4, Theorem 4 in \cite{ashraphijuo2017characterization2}, Lemma 9, Theorem 7 in \cite{ashraphijuo2016deterministic}).

\item[\ding{188}]
Deterministic conditions for unique completability in LCRTC (Lemma 18 in \cite{ashraphijuo2017fundamental}).

\item[\ding{189}]
An algorithm for Robust PCA that does not require coherence assumptions (Algorithm 1, Lemma 1 in \cite{pimentel2017random}).

\item[\ding{190}]
Deterministic conditions for unique recovery in Robust LRMC (Lemma 2, Theorem 2 in \cite{ashraphijuo2018deterministic}).

\end{itemize}

Fundamentally, each of the results above project incomplete data $\{\x^\o_\i\}$ onto candidate projections $\{\sU^\o_\i\}$, and rely on Theorem 1 in \cite{pimentel2015deterministic} to discard sets of projections that are incompatible with a single subspace (corresponding to incorrect solutions).

\begin{shaded}
The main limitation of Theorem 1 in \cite{pimentel2015deterministic} is that it requires exact (noiseless) projections, which are rarely available in practice. Consequently so do its byproducts, such as the results described in \ding{182}-\ding{190}. This paper extends Theorem 1 in \cite{pimentel2015deterministic} to allow noisy projections, and has the potential to enable the generalization of results like \ding{182}-\ding{190} to noisy settings, thus expanding their practical applicability.
\end{shaded}

To see this, observe that Theorem 1 in \cite{pimentel2015deterministic} is the special case of Theorem \ref{mainThm} where $\epsilon=0$. In that case, since $\sU$ is in general position, $\sigma(\B)>0$, the bound in \eqref{mainBound} simplifies to $\dist_\Gr(\sU,\sW) = 0$, showing that $\sW=\sU$, thus recovering Theorem 1 in \cite{pimentel2015deterministic}. More generally, Theorem 1 shows that if the {\em noisy} projections $\{\sV^\o_\i\}$ (properly sampled according to \textbf{C1}-\textbf{C3}) are close to the {\em true} projections $\{\sU^\o_\i\}$ (within $\epsilon$), then the optimal estimator $\sW$ will be close to the true subspace $\sU$ (within $\epsilon':=\epsilon \sqrt{2\r(\m-\r)} / \delta\sigma(\B)$.


\section{Proof}
\label{proofSec}
To obtain the bound in Theorem \ref{mainThm} we first show that if $\U^\o_\i$ is close to $\V^\o_\i$, then the chordal distance $\dist_\Gr(\sU^\o_\i,\sV^\o_\i)$ is small (Corollary \ref{projCor}). Next we show that this chordal distance between two subspaces is equal to that of their null-spaces (Lemma \ref{perpLem}). This will bound the error on the null-space bases $\b^\o_\i$ (Lemma \ref{abLem}), which are used to construct $\B$. Finally, we bound the chordal error of $\B$ (Corollary \ref{TotalCor}), and show that it is equal to that of the estimator $\sW$ using again Lemma \ref{perpLem}.

To bound the chordal distance between $\sU^\o_\i$ and $\sV^\o_\i$ we use the following well-known result from perturbation theory \cite{li2013new}:

\begin{myLemma}[Perturbation bound on the projection operator]
\label{projLem}
Let $\P_{\M_1}$ and $\P_{\M_2}$ denote the projection operators onto $\spn\M_1$, and $\spn\M_2$. Suppose $\rnk\M_1=\rnk\M_2$. Then
\begin{align*}
\| \P_{\M_1} - \P_{\M_2} \|_\Fr
\ \leq \
\sqrt{2} \ \|\M_1-\M_2\|_\Fr \min\{ \|\M_1^\dagger \| , \|\M_2^\dagger \| \},
\end{align*}
where $^\dagger$ denotes the Moore-Penrose inverse.
\end{myLemma}

Using Lemma \ref{projLem}, we can directly bound the chordal distance between $\sU^\o_\i$ and $\sV^\o_\i$ as follows:

\begin{myCorollary}
\label{projCor}
Let $\P_{\U^\o_\i}$ and $\P_{\V^\o_\i}$ denote the projection operators onto $\sU^\o_\i$ and $\sV^\o_\i$. Then
\begin{align*}
\| \P_{\U^\o_\i} - \P_{\V^\o_\i} \|_\Fr \ \leq \ \epsilon \sqrt{2\r} / \delta.
\end{align*}

\end{myCorollary}
\begin{proof}
By Lemma \ref{projLem}:
\begin{align*}
\| \P_{\U^\o_\i} - \P_{\V^\o_\i} \|_\Fr
\leq
\sqrt{2} \ \|\U^\o_\i-\V^\o_\i\|_\Fr \min\{ \|\U^{\o\dagger}_\i\|, \|\V^{\o\dagger}_\i\| \}.
\end{align*}
The corollary follows because
\begin{align*}
\|\U^\o_\i-\V^\o_\i\|_\Fr 
\ = \
\|\Z^\o_\i\|_\Fr 
\ \leq \
\sqrt{\r}\|\Z^\o_\i\|
\ \leq \
\epsilon \sqrt{\r},
\end{align*}
and because
\begin{align*}
\min\{ \|\U^{\o\dagger}_\i\|, \|\V^{\o\dagger}_\i\| \} 
\ \leq \
\|\V^{\o\dagger}_\i\| 
\ = \
1/\sigma(\V^\o_\i)
\ \leq \
1/\delta.
\end{align*}
where the last inequality follows by assumption \textbf{C5}.
\end{proof}

Recall that to construct our estimate $\sW$ we use the null-space of $\sV^\o_\i$, which is $\nu$-almost surely spanned by a single vector $\b^\o_\i$. To bound the error of $\b^\o_\i$, let $\a^\o_\i$ be a normal vector spanning the null-space of $\sU^\o_\i$ such that $\<\a^\o_\i,\b^\o_\i\> \geq 0$. In words, $\a^\o_\i$ is the noiseless version of $\b^\o_\i$. To bound the error between $\b^\o_\i$ and $\a^\o_\i$ we will use the following Lemma, which states that the chordal distance between any two subspaces is equal to that of their null-spaces.

\begin{myLemma}[Orthogonal complement projection bound]
\label{perpLem}
Let $\P^\perp_{\M_1}$ and $\P^\perp_{\M_2}$ denote the projection operators onto $\ker \M_1^\T$, and $\ker \M_2^\T$. Then
\begin{align*}
\| \P^\perp_{\M_1} - \P^\perp_{\M_2} \|_\Fr \ = \ \| \P_{\M_1} - \P_{\M_2} \|_\Fr.
\end{align*}
\end{myLemma}
\begin{proof}
Simply recall that $\P^\perp_{\M_1} = (\I-\P_{\M_1})$, where $\I$ denotes the identity matrix.
\end{proof}

Using this result, we can directly bound the error of $\b^\o_\i$:

\begin{myLemma}
\label{abLem}
$\nu$-almost surely, $\|\a^\o_\i-\b^\o_\i\| \leq \epsilon \sqrt{2\r} / \delta$.
\end{myLemma}
\begin{proof}
First notice that
\begin{align*}
\|\a^\o_\i-\b^\o_\i\|^2
\ &= \ \|\a^\o_\i\|^2-2\<\a^\o_\i,\b^\o_\i\> + \|\b^\o_\i\|^2 \\
\ &\leq \ \|\a^\o_\i\|^4 -2\<\a^\o_\i,\b^\o_\i\>^2 + \|\b^\o_\i\|^4 \\
\ &= \ \|\a^\o_\i\a^{\o\T}_\i-\b^\o_\i\b^{\o\T}_\i\|^2_\Fr,
\end{align*}
where the inequality follows because $\|\a^\o_\i\|=\|\b^\o_\i\|=1$ and $0 \leq \<\a^\o_\i,\b^\o_\i\> \leq 1$ by construction. Next recall that $\a^\o_\i$ and $\b^\o_\i$ are defined as the normal bases of $\ker \U^{\o\T}_\i$ and $\ker \V^{\o\T}_\i$, so
\begin{align*}
\|\a^\o_\i\a^{\o\T}_\i - \b^\o_\i\b^{\o\T}_\i\|_\Fr
\ &= \
\| \P^\perp_{\U^\o_\i} - \P^\perp_{\V^\o_\i} \|_\Fr \\
\ &= \
\| \P_{\U^\o_\i} - \P_{\V^\o_\i} \|_\Fr 
\ \leq \ \epsilon \sqrt{2\r} / \delta,
\end{align*}
where the last two steps follow directly from Lemma \ref{perpLem} and Corollary \ref{projCor}.
\end{proof}

Recall that our estimate $\sW$ is given by $\ker\B^\T$, where $\B=[\b_1,\dots,\b_\N]$, and $\b_\i \in \R^\m$ is equal to $\b^\o_\i$ in the coordinates of $\o_\i$, and zeros elsewhere. To bound the error of $\B$, define $\A=[\a_1,\dots,\a_\N]$, where $\a_\i \in \R^\m$ is equal to $\a^\o_\i$ in the coordinates of $\o_\i$, and zeros elsewhere. In words, $\A$ is the noiseless version of $\B$. Using Lemma \ref{abLem} we can directly bound the error between $\A$ and $\B$ as follows:

\begin{myCorollary}
\label{ABCor}
$\nu$-almost surely, $\|\A-\B\|_\Fr \leq \epsilon \sqrt{2\r(\m-\r)}/\delta $.
\end{myCorollary}
\begin{proof}
Since $\a_\i$ and $\b_\i$ are identical to $\a^\o_\i$ and $\b^\o_\i$, except filled with zeros in the same locations,
\begin{align*}
\|\A-\B\|_\Fr^2
\ &= \
\sum_{\i=1}^\N \|\a_\i-\b_\i\|_\Fr^2
\ = \
\sum_{\i=1}^\N \|\a^\o_\i-\b^\o_\i\|_\Fr^2 \\
\ &\leq \
\sum_{\i=1}^\N 2\r\epsilon^2/\delta^2 
\ = \ 2\r(\m-\r)\epsilon^2/\delta^2,
\end{align*}
where the last two steps follow by Lemma \ref{abLem}, and because $\N = \m-\r$ by assumption $\textbf{C2}$.
\end{proof}

At this point we can use again Lemma \ref{projLem} to bound the chordal error of $\B$ as follows:

\begin{myCorollary}
\label{TotalCor}
$\nu$-almost surely,
\begin{align*}
\|\P_\A-\P_\B\|_\Fr \ \leq \ \frac{\sqrt{2} \ \epsilon \sqrt{2\r(\m-\r)} }{\delta \sigma(\B)}.
\end{align*}
\end{myCorollary}

\begin{proof}
From Lemma \ref{projLem} and Corollary \ref{ABCor},
\begin{align*}
\|\P_\A-\P_\B\|_\Fr
\ &\leq \
\sqrt{2} \ \|\A-\B\|_\Fr \ \min\{ \|\A^\dagger\|,\|\B^\dagger\| \} \\
\ &\leq \
\sqrt{2} \ (\epsilon \sqrt{2\r(\m-\r)}/\delta) \ \min\{ \|\A^\dagger\|,\|\B^\dagger\| \} \\
\ &\leq \  \sqrt{2} \ (\epsilon \sqrt{2\r(\m-\r)}/\delta) \ \|\B^\dagger\|.
\end{align*}
\end{proof}

With this, we have a complete proof of our main result:

\begin{proof}[Proof (Theorem \ref{mainThm})]
By definition, $\sW=\ker\B^\T$, so $\P_\W = \P_\B^\perp$, and similarly $\P_\U=\P_\A^\perp$. We thus have:
\begin{align*}
\dist_\Gr(\sU,\sW) \ :&= \ \frac{1}{\sqrt{2}}\|\P_\U-\P_\W\|_\Fr
\ = \ \frac{1}{\sqrt{2}} \|\P_\A^\perp - \P_\B^\perp\|_\Fr \\
\ &= \ \frac{1}{\sqrt{2}} \|\P_\A - \P_\B\|_\Fr
\ \leq \ \frac{\epsilon \sqrt{2\r(\m-\r)} }{\delta \sigma(\B)},
\end{align*}
where the last line follows by Lemma \ref{perpLem} and Corollary \ref{TotalCor}.
\end{proof}

\section{Subspace Orientation and Sampling}
\label{orientationSec}


Theorem \ref{mainThm} shows that the error of the optimal estimate $\sW$ is bounded by the projection's noise-to-signal ratio $\epsilon/\delta$ scaled by a factor of $\sqrt{2\r(\m-\r)}$ related to the degrees of freedom of an $\r$-dimensional subspace of $\R^\m$, and a term $1/\sigma(\B)$. Recall that $\sigma(\B)$ denotes the smallest singular value of $\B$, and that $\B$ is constructed from the null-spaces of the projections $\sV^\o_\i$ (see \eqref{estimatorEq}). Consequently, $\sigma(\B)$ depends on the particular sampling of the projections and the orientation of $\sU$ in a very intricate way. To better understand this dependency, observe that $\sigma(\B)$ is determined by the joint orientation of its columns $\{\b_\i\}$. If the residual of any one column when projected onto the rest is small, then $\sigma(\B)$ will be small. Since $\b_\i$ is zero in the entries not in $\o_\i$, the residuals will strongly depend on the projection samplings $\{\o_\i\}$ and their overlaps. To better study this, let the $\i^{\rm th}$ column of $\O \in \{0,1\}^{\m \times \N}$ take the value $1$ in the rows in $\o_\i$, so that the $\i^{\rm th}$ column of $\O$ indicates the coordinates involved in the $\i^{\rm th}$ projection. This way, $\O$ also indicates the zero entries in $\B$. It is possible that the projections (and hence the columns of $\B$) share a large overlap. For example, consider the following {\em sampling pattern}:
\begin{align}
\label{sampling1Eq}
\O_1 \ = \ 
\left[ \begin{array}{c}
\hspace{.3cm} \Scale[1.5]{\boldsymbol{1}} \hspace{.3cm} \\ \hline
\\
\Scale[1.5]{\I} \\ \\
\end{array}\right]
\begin{matrix}
\left. \begin{matrix} \\ \end{matrix} \right\} \r \hspace{.7cm} \\
\left. \begin{matrix} \\ \\ \\ \end{matrix} \right\} \m-\r,
\end{matrix}
\end{align}
\noindent where $\boldsymbol{1}$ represents a block of all 1's. This sampling satisfies the identifiability conditions \textbf{C1}-\textbf{C3}, and has a {\em maximal} overlap (each projection shares $\r$ coordinates with one another; if two projections shared any more, they would violate condition \textbf{C3}). These large overlaps would result in more non-zero common entries between $\b_\i$ and $\b_\j$, potentially resulting smaller residuals. In contrast, consider the following sampling pattern:
\begin{center}
\includegraphics[width=4.25cm]{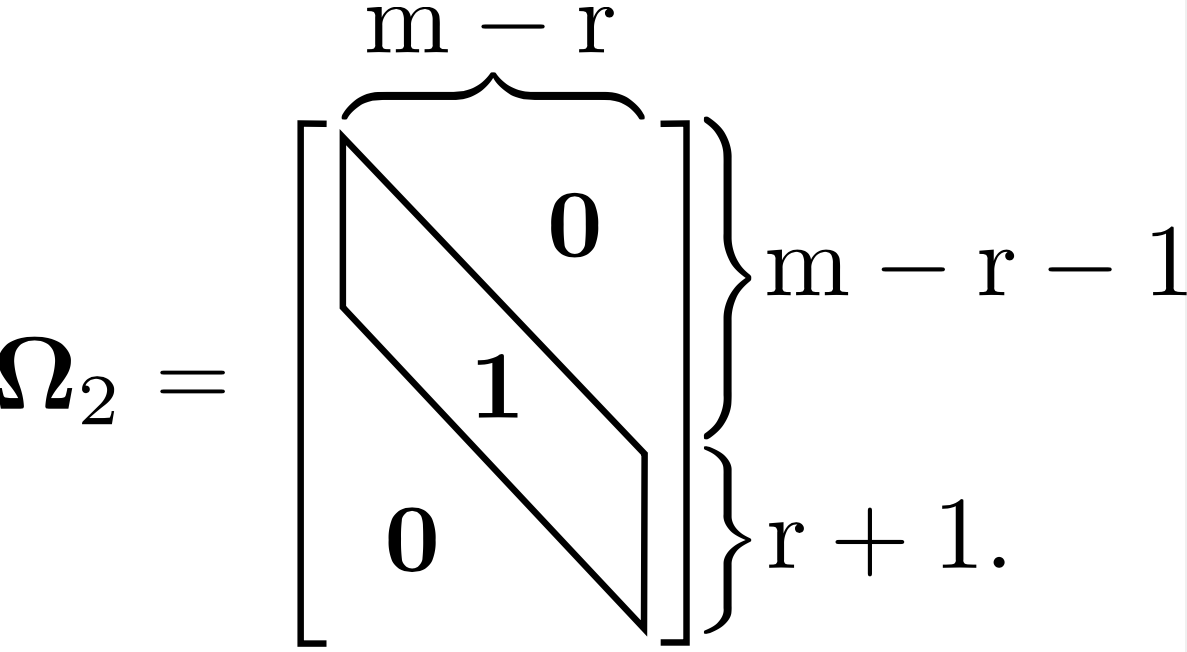}
\end{center}

\vspace{-2.25cm}
\begin{align}
\label{sampling2Eq}
\phantom{1}
\end{align}
\vspace{.45cm}

This sampling also satisfies the identifiability conditions \textbf{C1}-\textbf{C3}, but has fewer overlaps. For instance, the first and last columns share no coordinates whatsoever. Fewer overlaps will result in fewer non-zero common entries between $\b_\i$ and $\b_\j$, potentially resulting in larger residuals.

In conclusion, the sampling pattern will affect $\sigma(\B)$. However, the residuals do not depend on the sampling pattern alone, but also on the {\em partial} residuals of their overlapped (non-zero) entries, which in turn depend on the orientation of $\sU$. To see this recall that for every $\o_\i$, $\nu$-almost every subspace $\sU$ has a basis $\U^\o_\i$ in the following column-echelon form:
\begin{align*}
\U^\o_\i \ = \ \left[ \begin{array}{c}
\\ \Scale[2]{\I} \\ \\ \hline
\Scale[1.5]{\phantom{A}}{\al_\i^\T} \hspace{.2cm} \\
\end{array}\right]
\begin{matrix}
\left. \begin{matrix} \\ \\ \\ \end{matrix} \right\} \r \\
\left. \begin{matrix} \\ \end{matrix} \right\} 1,
\end{matrix}
\end{align*}
where $\al_\i \in \R^{\r}$ characterizes the specific projection $\sU^\o_\i$ and is completely arbitrary (in other words, $\al_\i$ could take any value, each resulting in a different projection $\sU^\o_\i$).
Notice that $\a_\i' := [\al_\i^\T, -1 ]^\T$ is in the null-space of $\sU^\o_\i$ (because $\U^{\o\T}_\i \a_\i'=0$), so $\b^\o_\i$ is essentially a small perturbation (within $\epsilon\sqrt{2\r (\m-\r)}/\delta$; see Corollary \ref{abLem}) of $\a^\o_\i=\a_\i'/\|\a_\i'\|$. In conclusion, $\sigma(\B)$ directly depends on $\al_\i,\al_\j$, which characterize the projections $\sU^\o_\i,\sU^\o_\j$, and are completely arbitrary.

Finally, recall that $\B$ is an $\m \times (\m-\r)$ matrix. As the gap $(\m-\r)$ between the ambient dimension and the subspace dimension grows, $\B$ has more columns, and the likelihood of obtaining one small residual increases \cite{cai2013distributions} \---- except perhaps if, for example, each column of $\B$ has a unique coordinate, as is the case with $\O_1$ in \eqref{sampling1Eq}. Ultimately, a smaller residual will result in a smaller $\sigma(\B)$, and a looser bound. This is verified in our experiments, where our bound becomes looser as the ambient dimension $\m$ increases away from $\r$.

To summarize, $\sigma(\B)$ encodes the intricate dependency of our bound on the particular sampling of the projections and on the particular orientation of $\sU$, all in a single term that can be easily and directly computed from the observed data.

\section{Experiments}

\begin{figure*}
    \centering
    \includegraphics[width=\textwidth]{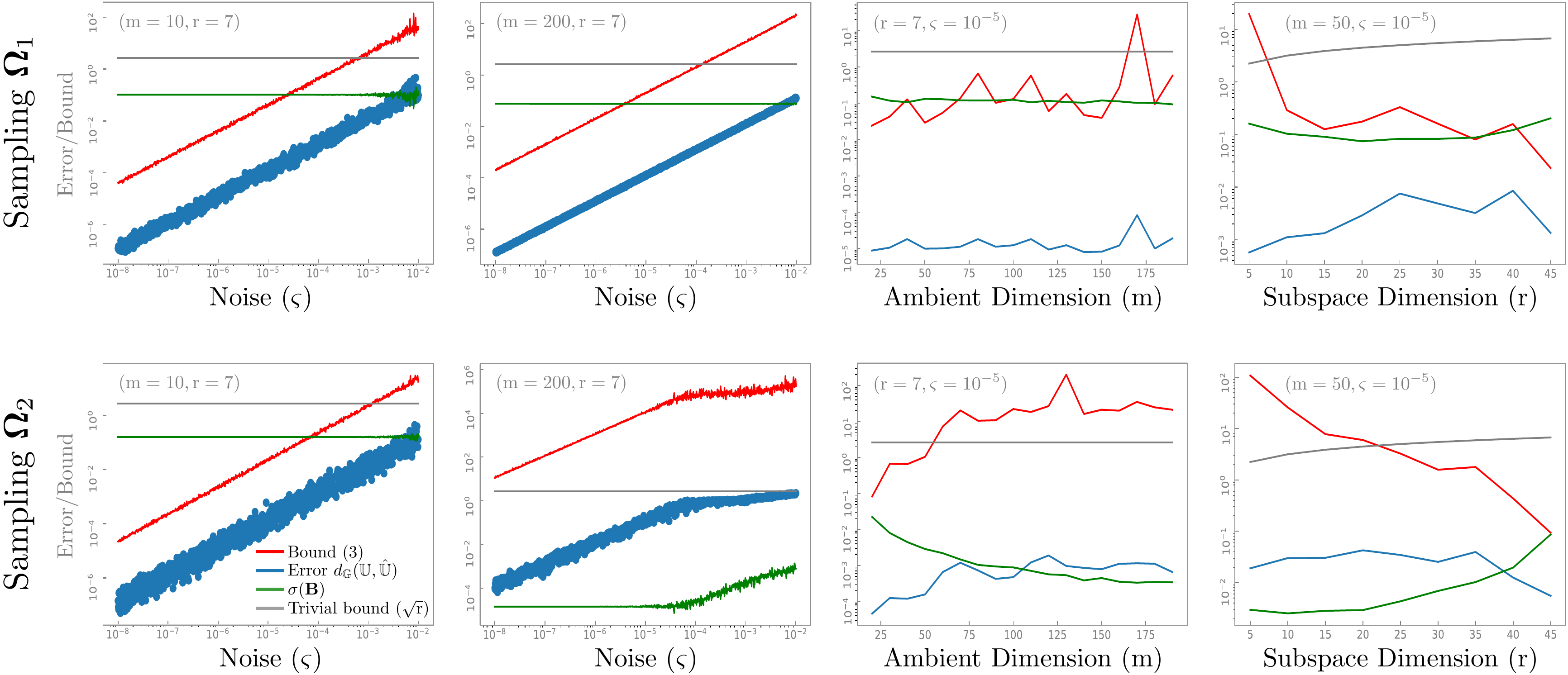}
    \caption{Our bound behaves nicely as a function of noise, but becomes looser as the term $\sigma(\B)$ shrinks when the dimension gap $(\m-\r)$ grows. Nonetheless, in all cases our bound closely follows the exact error except for a constant gap. Notice that both error and bound depend heavily on the observation pattern. Notice the significantly smaller errors and bounds obtained with sampling $\O_1$. Code available at \href{https://github.com/ksrivastava1/identifying-subspaces }{https://github.com/ksrivastava1/identifying-subspaces} .} 
    \label{fig:error_noise}
    \vspace{-.25cm}
\end{figure*}

In this section we present numerical simulations to further analyze our bound from Theorem  \ref{mainThm} as a function of the noise level $\epsilon$, the ambient dimension $\m$, and the subspace dimension $\r$. As discussed in Section \ref{orientationSec}, our bound also depends on a term $1/\sigma(\B)$, which encodes the information about the particular sampling pattern, and the orientation of $\sU$. To analyze this dependency we will study two sampling patterns, namely $\O_1$ and $\O_2$ as depicted in \eqref{sampling1Eq} and \eqref{sampling2Eq}, corresponding to two extremes of valid sampling overlaps.

In each experiment, we generated an $\m \times \r$ basis $\U$ with i.i.d.~random entries drawn from a standard normal distribution. Then for each $\i=1,\dots,\N$, we generated an $(\r+1) \times \r$ noise matrix $\Z^\o_\i$ with i.i.d.~random entries drawn from a normal distribution with zero mean and variance $\varsigma^2$. We then generated $\V^\o_\i$ according to \eqref{vEq}. Finally, we computed our bound parameters $\epsilon = \max_\i{\|\Z^\o_i\|}$, and $\delta=\min_\i {\sigma(\V^\o_\i)}$. In our experiments we numerically compare the exact subspace estimation error $\dist_\Gr(\sU,\sW)$ with our bound, highlighting changes in $\delta$ and $\sigma(\B)$, together with the upper bound of the chordal distance, given by $\sqrt{r}$.

\vspace{0.1cm}
\noindent \textbf{Effect of the noise.} In our first experiment we investigate how our bound from Theorem \ref{mainThm} changes as a function of the noise level $\varsigma$, with $\m=\{10,200\}$ and $\r=7$ fixed. In each of $1000$ independent trials we selected $\varsigma$ uniformly at random in the range $(10^{-8},10^{-2})$. The results are in Figure \ref{fig:error_noise}. They show that the optimal subspace estimator $\sW$ is not too sensitive to noise, and that our bound from Theorem \ref{mainThm} can be tight if the gap $(\m-\r)$ between the ambient dimension and the subspace dimension is small. However, the bound becomes much looser if this gap increases. This is an artifact of the $1/\sigma(\B)$ term in our bound. As discussed in Section \ref{orientationSec}, this term does not change much with noise, but does change with the dimension gap and the sampling pattern (notice the significantly smaller error and bound obtained with sampling $\O_1$). We further explore this dependency in our next experiments.


\vspace{0.1cm}
\noindent \textbf{Effect of the Ambient Dimension.} In our second experiment we study how our bound from Theorem \ref{mainThm} changes as a function of the ambient dimension $m$, with $\r=7$ and $\varsigma=10^{-5}$ fixed. The results are summarized in Figure \ref{fig:error_noise}, where each point shows the average over $100$ independent trials. Notice that the bound is extremely tight for small $\m$, and slowly becomes looser as $\m$ increases. This is primarily due to the decrease in the smallest singular value of $\B$, which as discussed in Section \ref{orientationSec}, depends heavily on the sampling pattern $\O$ and tends to become smaller as $\m$ grows \cite{cai2013distributions}. Notice the significantly different results produced by changing nothing but the sampling pattern. While both samplings $\O_1, \O_2$ in \eqref{sampling1Eq} are {\em valid} (i.e., they satisfy the identifiability conditions \textbf{C1}-\textbf{C3}), the errors and bounds produced by each sampling are notoriously different.







\vspace{0.1cm}
\noindent \textbf{Effect of the Subspace Dimension} In our final experiment, we study how our bound from Theorem \ref{mainThm} changes as a function of the subspace dimension $r$, with $m = 50$ and $\varsigma = 10^{-5}$ fixed. The results are summarized in Figure \ref{fig:error_noise}, where each point shows the average over $100$ independent trials. Consistent with our previous experiments, both error and bound depend heavily on the sampling pattern. Consistent with the previous experiment and with our discussion in Section \ref{orientationSec}, our bound becomes looser whenever the {\em gap} $(\m-\r)$ between the ambient dimension and the subspace dimension grows. We find it interesting that the exact error also increases with the same pattern. We conjecture that this is because our estimator relies on $\sU^\perp$, which becomes smaller and hence easier to estimate as this gap decreases.


\section{Conclusions}
In this paper we derive an upper bound for the optimal subspace estimator obtained from noisy versions of its canonical projections. This contribution generalizes a fundamental result (Theorem 1 in \cite{pimentel2015deterministic}) with important applications in matrix completion, subspace clustering, and related problems. Unfortunately, the term $\sigma(\B)$ decreases as the gap between the ambient and subspace dimensions grow, resulting in a looser bound. Our future work will investigate alternative, tighter bounds that do not rely on $\sigma(\B)$, and that degrade nicely as this dimension gap grows.



\bibliographystyle{IEEEtran}
\bibliography{main}

\end{document}